\documentclass[sigconf]{acmart}

\setcopyright{acmcopyright}
\copyrightyear{2020}
\acmYear{2020}

\usepackage{xcolor}
\usepackage[shortlabels]{enumitem}
\setlist[enumerate]{nosep}
\usepackage{multirow}
\usepackage{float}
\usepackage{graphicx}
\usepackage{float}
\usepackage{array}
\usepackage{amsmath,lipsum}
\newlength\Origarrayrulewidth
\newcommand{\Cline}[1]{%
  \noalign{\global\setlength\Origarrayrulewidth{\arrayrulewidth}}%
  \noalign{\global\setlength\arrayrulewidth{1.1pt}}\cline{#1}%
  \noalign{\global\setlength\arrayrulewidth{\Origarrayrulewidth}}%
}

\makeatletter
\newcommand\footnoteref[1]{\protected@xdef\@thefnmark{\ref{#1}}\@footnotemark}
\makeatother
\DeclareMathOperator*{\bigpar}{\bigg\Vert}
\DeclareMathOperator*{\nompar}{\Vert}
\DeclareMathOperator*{\bilstm}{\textit{BiLSTM}}
\DeclareMathOperator*{\mlp}{\textit{MLP}}
\renewcommand{\footnotesize}{\fontsize{8pt}{6pt}\selectfont}
\usepackage{graphicx}
\graphicspath{ {./images/} }
\usepackage[ruled]{algorithm2e}
\newtheorem{theorem}{Theorem}[section]
\newtheorem{lemma}{Lemma}[section]
\makeatletter
\newcommand{\vo}{\vec{o}\@ifnextchar{^}{\,}{}}
\makeatother
\usepackage{subcaption} 

\setcopyright{rightsretained}

\begin{document}
\copyrightyear{2020}  
\acmYear{2020}  
\acmConference[WWW '21]{TheWebConf 2021: ACM The Web Conference 2021}{April 19--23, 2021}{Ljubljana, Slovenia}
\acmBooktitle{WWW '21: The Web Conference,
 April 19--23, 2021, Ljubljana, Slovenia}

\settopmatter{authorsperrow=4}
\title{RECON: Relation Extraction using Knowledge Graph Context in a Graph Neural Network}


\author{Anson Bastos}
\email{cs20resch11002@iith.ac.in}
\affiliation{%
  \institution{IIT, Hyderabad and Zerotha Research}
  \country{India}
}

\author{Abhishek Nadgeri}
\email{abhishek22596@gmail.com }

\affiliation{%
  \institution{Zerotha Research and RWTH Aachen}
  \country{Germany}
}

\author{Kuldeep Singh}
\email{kuldeep.singh1@cerence.com}
\affiliation{%
  \institution{Zerotha Research and Cerence GmbH}
   \country{Germany}
}

\author{Isaiah Onando Mulang'}
\email{mulang@iai.uni-bonn.de}
\affiliation{%
  \institution{Zerotha Research and Fraunhofer IAIS, Germany}
}

\author{Saeedeh Shekarpour}
\email{sshekarpour1@udayton.edu}
\affiliation{%
  \institution{University of Dayton, USA}
}

\author{Johannes Hoffart}
\email{johannes.hoffart@gs.com}
\affiliation{%
   \institution{Goldman Sachs, Germany}
}

\author{Manohar Kaul}
\email{mkaul@iith.ac.in}
\affiliation{%
  \institution{IIT Hyderabad, India}
}
\renewcommand{\shortauthors}{Bastos, et al.}
\begin{abstract}
In this paper, we present a novel method named RECON, that automatically identifies relations in a sentence (sentential relation extraction) and aligns to a knowledge graph (KG). RECON uses a graph neural network to learn representations of both the sentence as well as facts stored in a KG, improving the overall extraction quality. These facts, including entity attributes (label, alias, description, instance-of) and factual triples, have not been collectively used in the state of the art methods. We evaluate the effect of various forms of representing the KG context on the performance of RECON. The empirical evaluation on two standard relation extraction datasets shows that RECON significantly outperforms all state of the art methods on NYT Freebase and Wikidata datasets.

\end{abstract}




\maketitle

\section{Introduction}

\begin{figure*}
	\centering
	\includegraphics[width=\textwidth]{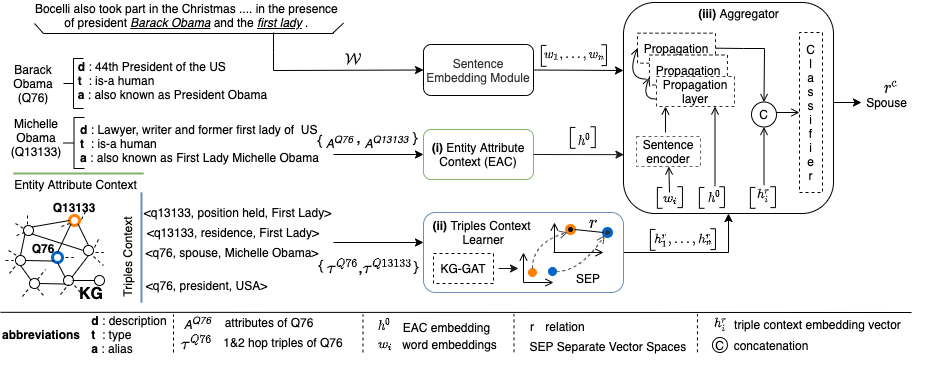}
	\caption{RECON has three building blocks: i) entity attribute context (EAC) encodes context from entity attributes ii) triple context learner independently learns relation and entity embeddings of the KG triples in separate vector spaces iii) a context aggregator (a GNN model) used for consolidating the KG contexts to predict target relation.}
	\label{fig:approach}
	    \vspace{-2mm}
	\end{figure*}

The publicly available Web-scale knowledge graphs (KGs) (e.g., DBpedia \cite{DBLP:conf/semweb/AuerBKLCI07}, Freebase \cite{DBLP:conf/aaai/BollackerCT07}, and Wikidata \cite{DBLP:conf/www/Vrandecic12}) find wide usage in many real world applications such as question answering, fact checking, voice assistants, and search engines \cite{fensel2020we}. Despite the success and popularity, these KGs are not exhaustive. Hence there is a need for approaches that automatically extract knowledge from unstructured text into the KGs \cite{DBLP:conf/aaai/LinLSLZ15}. Distantly supervised \emph{relation extraction (RE)} is one of the knowledge graph completion tasks aiming at determining the entailed relation between two given entities annotated on the text to a background KG \cite{DBLP:conf/aaai/WuFZ19}. For example, given the sentence \textit{"Bocelli also took part in the Christmas in Washington special on Dec 12, in the presence of president \underline{Barack Obama} and the \underline{first lady}"} with annotated entities- \textit{wdt:Q76 (Barack Obama)}\footnote{wdt: binds to \url{https://www.wikidata.org/wiki/}} and \textit{wdt:Q13133(Michelle Obama)};
the RE task aims to infer the semantic relationship. Here \textit{wdt:P26 (spouse)} is the target relation. In this example, one can immediately see the impact of background knowledge: the correct target relation \emph{spouse} is not explicitly stated in the sentence, but given background knowledge about the first lady and her marital status,  the correct relation can be inferred by the model. In cases having no relations, the label ``NA'' is predicted.

Existing RE approaches have mainly relied on the 
multi-instance and distant learning paradigms  \cite{smirnova2018relation}.
Given a bag of sentences (or instances), the \textbf{multi-instance RE} considers all previous occurrences of a given entity pair while predicting the target relation \cite{DBLP:conf/emnlp/VashishthJPBT18}. However, incorporating contextual signals from the previous occurrences of entity pair in the neural models add some noise in the training data, resulting in a negative impact on the overall performance \cite{liu2017soft}. Several approaches (e.g., based on attention mechanism \cite{DBLP:conf/naacl/YeL19}, neural noise converter \cite{DBLP:conf/aaai/WuFZ19}) have been proposed to alleviate the noise from the previous sentences for improving overall relation extraction. 
Additionally, to mitigate the noise in multi-instance setting, there are few approaches that not only use background KGs as a source of target relation but exploit specific properties of KGs as additional contextual features for augmenting the learning model.
Earlier work by \cite{DBLP:conf/aaai/Ji0H017,DBLP:conf/emnlp/VashishthJPBT18} utilizes entity descriptions and entity/relation aliases from the underlying KG as complementary features. 
Work in \cite{DBLP:conf/acl/NathaniCSK19} employs attention-based embeddings of KG triples to feed in a graph attention network for capturing the context.
Overall, the knowledge captured from KG complements the context derived from the text.

In contrast, the \textbf{sentential RE} \cite{DBLP:conf/emnlp/SorokinG17} ignores any other occurrence of
the given entity pair, thereby making the target relation predictions on the sentence level. However, the existing approaches for sentential RE \cite{DBLP:conf/emnlp/SorokinG17,DBLP:conf/acl/ZhuLLFCS19} rely on local features/context present in the sentence and do not incorporate any external features. In this paper, we study the effect of KG context on sentential RE task by proposing a novel method RECON. 
Our approach focuses on an effective representation of the knowledge derived from the KG induced in a graph neural network (GNN). The proposed approach has three building blocks illustrated in the Figure \ref{fig:approach}. Specifically, RECON harnesses the following three novel insights to outperform existing sentential and multi-instance RE methods:

\begin{itemize}
  \item \textit{Entity Attribute Context}: we propose a recurrent neural network based module that learns representations of the given entities expanded from the KG using entity attributes (properties) such as \textit{entity label}, \textit{entity alias}, \textit{entity description} and \textit{entity Instance of (entity type)}. 
  \item \textit{Triple Context Learner}: we aim to utilize a graph attention mechanism to capture both entity and relation features in a given entity's multi-hop neighborhood. By doing so, our hypothesis is to supplement the context derived from the previous module with the additional neighborhood KG triple context. For the same, the second module of RECON independently yet effectively learns entity and relation embeddings of the 1\&2-hop triples of entities using a graph attention network (GAT) \cite{velivckovic2017graph}. 
  \item \textit{Context Aggregator}: our idea is to exploit the message passing capabilities of a graph neural network \cite{DBLP:conf/acl/ZhuLLFCS19} to learn representations of both the sentence and facts stored in a KG. Hence, in the third module of RECON, we employ an aggregator consisting of a GNN and a classifier. It receives as input the sentence embeddings, entity attribute context embeddings, and the triple context embeddings. The aggregator then obtains a homogeneous representation, passed into a classifier to predict the correct relation.
\end{itemize}

We perform exhaustive evaluation to understand the efficacy of RECON in capturing the KG context. Our work has following contributions:
\begin{itemize}
\item RECON: a sentential RE approach that utilizes entity attributes and triple context derived from the Web scale knowledge graphs, induced in a GNN, thereby significantly outperforming the existing baselines on two standard datasets.
\item We augment two datasets: Wikidata dataset \cite{DBLP:conf/emnlp/SorokinG17} and NYT dataset for Freebase \cite{DBLP:conf/pkdd/RiedelYM10} with KG context. Our implementation and datasets are publicly available\footnote{\url{https://github.com/ansonb/RECON}}.
   
\end{itemize}
The structure of the paper is follows: Section \ref{sec:related} reviews the related work. Section 3 formalizes the problem and the proposed approach. Section \ref{sec:experiment} describes experiment setup. Our results and ablation studies are illustrated in Section \ref{sec:evaluation}. We conclude in Section \ref{sec:conclusion}.

\section{Related Work} \label{sec:related}
\par \textbf{Multi-instance RE:} The recent success in RE can attribute to the availability of vast training data curated using distant supervision \cite{DBLP:conf/acl/MintzBSJ09}. Methods for distant supervision assume that if two entities have a relationship in a KG, then all sentences containing those entities express the same relation, this may sometimes lead to noise in the data. To overcome the challenges, researchers in \cite{DBLP:conf/pkdd/RiedelYM10} initiated the multi-instance learning followed by \cite{DBLP:conf/acl/HoffmannZLZW11} which extracted relation from a
bag of sentences. 
Researchers \cite{DBLP:conf/aaai/Ji0H017} attained improved performance by introducing entity descriptions as KG context to supplement the task. The RESIDE approach \cite{DBLP:conf/emnlp/VashishthJPBT18} ignores entity descriptions but utilize entity type along with relation and entity aliases. RELE approach \cite{DBLP:conf/emnlp/HuZSNGY19} jointly learned embeddings of structural information from KGs and textual data from entity descriptions to improve multi-instance RE. Unlike existing approaches where one or other entity attributes are considered, in this work, we combined four typical properties of KG entities for building what we refer as entity attribute context. \\
\textbf{Learning information from KG Triples}: The survey \cite{wang2017knowledge} provides holistic overview of available KG embedding techniques and their application in entity oriented tasks. TransE \cite{DBLP:conf/nips/BordesUGWY13} studied knowledge base completion task using entity and relation embeddings learned in the same vector space. It lacks ability to determine one-to-many, many-to-one, and many-to-many relations. TransH \cite{DBLP:conf/aaai/WangZFC14} has tried to address this problem by learning embeddings on different hyperplanes per relation. However, the entity and relation embeddings are still learned in the same space. TransR \cite{DBLP:conf/aaai/LinLSLZ15} represents entity and relation embeddings in separate vector spaces, which works better on the task of relation prediction and triple classification. They perform a linear transformation from entity to relation embedding vector space. Work by \cite{DBLP:conf/naacl/XuB19} and \cite{DBLP:conf/aaai/Han0S18} are few attempts for jointly learning different representations from text and facts in an existing knowledge graph.
Furthermore, graph attention network (GAT) has been proposed to learn embeddings for graph-structured data \cite{velivckovic2017graph}. KBGAT is an extension of GAT that embeds KG triples by training entities and relations in same vector spaces specifically for relation prediction \cite{DBLP:conf/acl/NathaniCSK19}. However, we argue that entity and relation embedding space should be separated. Moreover, the transformation from entity to relation space should be nonlinear and distinct for every relation. This setting allows the embeddings to be more expressive (section \ref{sec:evaluation}).\\
\textbf{Sentential RE}: There exists a little work on the sentential RE task. The work in \cite{DBLP:conf/emnlp/SorokinG17} established an LSTM-based baseline that learns context from other relations in the sentence when predicting the target relation. \cite{DBLP:conf/acl/ZhuLLFCS19} generate the parameters of graph neural networks (GP-GNN) according to natural language sentences for multi-hop relation reasoning for the entity pair.  

\section{Problem Statement and Approach} \label{sec:problem}
\subsection{Problem Statement}
We define a KG as a tuple $KG = (\mathcal{E},\mathcal{R},\mathcal{T}^+)$ where $\mathcal{E}$ denotes the set of entities (vertices), $\mathcal{R}$ is the set of relations (edges), and $\mathcal{T}^+ \subseteq \mathcal{E} \times \mathcal{R} \times \mathcal{E} $ is a set of all triples. 
A triple $\uptau = (e_h,r,e_t) \in \mathcal{T}^+$ indicates that, for the relation $r \in \mathcal{R}$, $e_h$ is the head entity (origin of the relation) while $e_t$ is the tail entity. Since $KG$ is a multigraph; $e_h=e_t$ may hold and $|\{r_{e_h,e_t}\}|\geq 0$ for any two entities. We define the tuple $(A^e,\tau^e) = \varphi(e)$ obtained from a context retrieval function $\varphi$, that returns, for any given entity $e$, two sets: $A^e$, a set of all attributes and $\tau^e \subset \mathcal{T}^+$ the set of all triples with head at $e$.

A sentence $\mathcal{W} = ( w_1,w_2,...,w_l )$ is a sequence of words.
The set of entities in a sentence is denoted by $\mathcal{M} = \{m_1,m_2,...,m_k\}$ where every $m_k=(w_i,...,w_j)$ is a segment of the sentence $\mathcal{W}$.
Each mention is annotated by an entity from KG as $[m_i:e_j]$ where $e_j \in \mathcal{E}$. Two annotated entities form a pair $p = \langle e_b,e_l\rangle$ when there exists a relationship between them in the sentence (in case no corresponding relation in the KG - label N/A).

The \textit{RE Task} predicts the target relation $r^c \in \mathcal{R}$ for a given pair of entities $\langle e_i,e_j\rangle$ within the sentence $\mathcal{W}$.
If no relation is inferred, it returns 'NA' label.
We attempt the \underline{sentential RE} task which posits that the sentence within which a given pair of entities occurs is the only visible sentence from the bag of sentences. 
All other sentences in the bag are not considered while predicting the correct relation $r^c$. Similar to other researchers~\cite{DBLP:conf/emnlp/SorokinG17}, we view RE as a classification task. However, we aim to model KG contextual information to improve the classification. This is achieved by learning representations of the sets $A^e,\tau^e$,  and $\mathcal{W}$ as described in section \ref{sec:approach}.

\subsection{RECON Approach} \label{sec:approach}
Figure \ref{fig:approach} describes the RECON architecture. The sentence embedding module retrieves the static embeddings of the input sentence. The EAC module (sec. \ref{sec:eac}) takes each entity of the sentence and enrich the entity embeddings with corresponding contextual representation from the KG using entity properties such as aliases, label, description, and instance-of. The Triple context learner module (sec. \ref{sec:triple}) learns a representation of entities and relations in a given entity’s 2-hop neighborhood. A Graph Neural Network is finally used to aggregate the entity attribute, KG triple, and sentence contexts with a relation classification layer generating the final output (sec. \ref{sec:aggre}).
We now present our approach in detail.
\subsubsection{ Entity Attribute Context (EAC)} \label{sec:eac}
The entity attribute context is built from commonly available properties of a KG entity \cite{hogan2020knowledge}:  \textit{entity labels}, \textit{entity alias}, \textit{entity description}, and \textit{entity Instance of}. We extract this information for each entity from the public dump of Freebase \cite{DBLP:conf/aaai/BollackerCT07}, and Wikidata \cite{DBLP:conf/www/Vrandecic12}) depending on the underlying KG (cf. section \ref{sec:experiment}). To formulate our input, we consider the literals of the retrieved entity attributes. For each of these attributes, we concatenate the word and character embeddings and pass them through a bidirectional-LSTM encoder \cite{schuster1997bidirectional}. The final outputs from the BiLSTM network are stacked and given to a one dimensional convolution network (CNN) described in the Figure \ref{fig:entity-attribute} and formalized in equation \ref{eq_1}. The reasons behind choosing CNN are i) to enable a dynamic number of contexts using the max pooling ii) to keep the model invariant of the order in which the context is fed.
\begin{equation} \label{eq_1}
h^o = \textrm{1D\_CNN}(\nompar_{i=0}^{N} [\bilstm (A_i)] )
\end{equation}
where each $A_i$ is attribute of given entity and $\parallel$ is the concatenation.
\begin{figure}[h!]
  \includegraphics[scale=0.40]{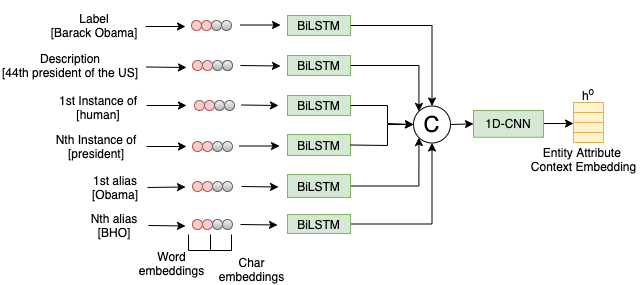}
  \caption{Entity Attribute Context Module}
  \label{fig:entity-attribute}
\end{figure}
\subsubsection{Triple Context Learner (KGGAT-SEP)} \label{sec:triple}
The \textit{KG triple context learner (KGGAT-SEP)} is an extension of KBGAT \cite{DBLP:conf/acl/NathaniCSK19} that retains the capability to capture context from neighboring triples in the KG. In addition, our idea is to learn the entity and relation embeddings of the triples in separate vector spaces to capture more expressive representations.
This is because each entity might be engaged in several relations in various contexts, and different aspects of the entity may participate in representing each relation \cite{DBLP:conf/aaai/LinLSLZ15}.
Let $\vec{e}_h$ and $\vec{e}_t$ be the initial entity vectors and $\vec{r}_k$ 
be a initial relation vector between them representing the triple $\uptau_{htk}$, 
$W$ is the weight metric, then the vector representation of triple is
\begin{equation}\label{eq_triple_kbgat}
 \small
    \vec{\uptau}_{htk} = W[\vec{e_h} \nompar \vec{e_t} \nompar \vec{r_k}]
\end{equation}
where we concatenate the head and tail entity embeddings and relation embedding vector. 
The importance of each triple (i.e. attention values) is represented by $b_{htk}$ and is computed as in equation \ref{eq_imp_kbgat} where \textit{LeakyReLU} is an activation function:
\begin{equation}\label{eq_imp_kbgat}
\small
b_{htk} = LeakyReLU(W_2\vec{\uptau}_{htk})
\end{equation}
To get the relative attention values over the neighboring triples, a softmax is applied to equation
\ref{eq_imp_kbgat}
\begin{equation}\label{eq_softmax_kbgat}
 \small
\alpha_{htk} = softmax_{tk}(b_{htk})\\
        =\frac{exp(b_{htk})}{\sum_{t \in N_{h}}\sum_{r \in R_{ht}}exp(b_{htr})}
\end{equation}
$N_h$ denotes the neighborhood of entity $e_h$ and $R_{ht}$ denotes the set of relations between entities $e_h$ and $e_t$.
The new embedding for the entity $e_h$ is now the weighted sum of the triple embeddings using equations \ref{eq_triple_kbgat} and \ref{eq_softmax_kbgat}. In order to stabilize the learning and encapsulate more information, X independent attention heads have been used and the final embedding is the concatenation of the embedding from each head:
\begin{equation}\label{eq_gat}
 \small
    \vec{e}_h^{~'} = \bigpar_{x=1}^{X} \sigma \left( \sum_{t \in N_{h}}\sum_{r \in R_{ht}}\alpha_{htk}^x\vec{\uptau}_{htk}^{~x} \right) 
\end{equation}
The original entity embedding $\vec{e}_h$ after a transformation, using matrix $W^E$, is added to the equation \ref{eq_gat} to preserve the initial entity embedding information. 
\begin{equation}\label{eq_gat_2}
\small
    \vec{e}_h^{~''} = \vec{e}_h^{~'} + W^{E}\vec{e}_h
\end{equation}
For relation embeddings, a linear transformation is performed on the initial embedding vector, using matrix $W^R$, to match the entity vector's dimension in equation \ref{eq_gat_2}
\begin{equation}\label{eq_gat_r}
    \vec{r}_k^{~''} = W^{R}\vec{r}_k
\end{equation}
Traditionally, the training objective for learning embeddings in same vector spaces are borrowed from \cite{DBLP:conf/nips/BordesUGWY13}. The embeddings here are learned such that, for a valid triple $\uptau_{htk}=(e_h,r_k,e_t)$ the following equation holds where $ \vec{e}_i^{~''}$ is embeddings in entity space.
\begin{equation}\label{eq_same_space_objective}
    \vec{e}_h^{~''}+\vec{r}_k^{~''}=\vec{e}_t^{~''}
\end{equation}
The optimization process tries to satisfy equation \ref{eq_same_space_objective} and the vectors are learned in \textit{same vector space}.
Contrary to the previous equation, we keep entities and relation embeddings in \textbf{separate spaces}. With that, we now need to transform entities from entity spaces to the relation space. We achieve this by applying a nonlinear transformation:
(cf. the theoretical foundation is in the section \ref{sec:theory}).
\begin{equation}\label{eq_2}
    \vec{e}_i^{~r}=\sigma\left(W_r\vec{e}_h^{~''}\right)
\end{equation}
here $\vec{e}_i^{~r}$ 
(where $i=\{h,t\}$) is the relation specific entity vector in the relation embedding space, $W_r$ is the relation specific transformation matrix and $\vec{e}_h^{~''}$ is the corresponding embedding in the entity space from equation \ref{eq_gat_2}. 
We presume that such separation helps to capture a comprehensive representations for relations and entities. Equation \ref{eq_same_space_objective} is now modified as
\begin{equation}\label{eq_sep_space_objective}
\small
    \vec{e}_h^{~r}+\vec{r}_k^{~''}=\vec{e}_t^{~r}
\end{equation}
We define a distance metric $d_{ht}$ for a relation $\vec{r}_k^{~''}$, representing the triple $\uptau_{htk}$ as
\begin{equation}\label{eq_d} 
\small
    d_{\uptau_{htk}} = \vec{e}_h^{~r}+\vec{r}_k^{~''}-\vec{e}_t^{~r}
\end{equation}
A margin ranking loss minimizes the following expression
\begin{equation}\label{eq_mrl}
 \small
    L(\Omega) = \sum_{\uptau_{ht} \in \mathcal{T}^{pos}}\sum_{\uptau_{ht}^{'} \in \mathcal{T}^{neg}} max\{ d_{\uptau_{ht}^{'}} - d_{\uptau_{ht}} + \gamma, 0 \}
\end{equation}
where $\mathcal{T}^{pos}$ is the set of valid triples, $\mathcal{T}^{neg}$ is the set of invalid triples and $\gamma$ is a margin parameter. We consider the actual triples present in the dataset as positive (valid) triples and the rest of the triples, which are not in the dataset as invalid. For example, as we do RE, if in the KG, entities Barack Obama and Michelle Obama have one valid relation "spouse," then the valid triple is <Barack Obama, spouse, Michelle Obama>. The invalid triples will contain relations that do not exist between these two entities.
\subsubsection{Aggregating KG Context} \label{sec:aggre}
For aggregating context from previous two steps, we adapt and modify generated parameter graph neural network (GP-GNN) \cite{DBLP:conf/acl/ZhuLLFCS19} due to its proven ability to enable message passing between nodes. It consists of an encoder module, a propagation module and a classification module. The encoder takes as input the word vectors concatenated to the position vectors from the sentence.
\begin{equation}
    E(w_s^{i,j}) = w_s \nompar p_s^{i,j}  
\end{equation}

where $p$ denotes the position embedding of word spot "s" in the sentence relative to the entity pair's position $i, j$ and $w$ is the word embedding. Position vectors are basically to mark whether the token belongs to head or tail entity or none of them. We use position embedding scheme from \cite{DBLP:conf/acl/ZhuLLFCS19}. We use concatenated word embeddings in a biLSTM followed by a fully connected network for generating transition matrix given as:
\begin{equation}\label{eq_11}
\begin{split}
B_{(i,j)} = [~ \mlp(~\bilstm\limits_{lyr=0}^{n-1}(~E(w_s^{i,j})_{_{s=1}}^{^{l}}~) ~]
\end{split}
\end{equation}
Here $[.]$ denotes conversion of vectors into a matrix, $lyr$ is the layer of biLSTM, $s$ is the index of word in sentence and $l$ is the length of the sentence. For each layer (n) of the propagation module we learn a matrix $B_{i,j}^{(n)}$ using equation \ref{eq_11}. Then, the propagation module learns representations of entity nodes $v$ (layer wise) according to the following equation
\begin{equation}
    h_i^{(n+1)} = \sum_{v_j \in N(v_i)} \sigma(B_{i,j}^{(n)} h_j^{(n)} )
\end{equation}
$N(v_i)$ represents the neighborhood of $v_i$.
Here $h^0$ is the initial entity embedding which is taken from equation \ref{eq_1}. In classification module, the vectors learned by each layer in the propagation module are concatenated and used for linking the relation:
\begin{equation}\label{eq_0}
    r_{v_i,v_j} = \nompar_{i=1}^{N}[h_{v_i}^{(i)} \odot h_{v_j}^{(j)}]^{\top}
\end{equation}
where $\odot$ denotes element wise multiplication. 
We concatenate the entity embeddings learned from the triples context in equation \ref{eq_gat_2} to $r_{v_i,v_j}$ obtained from \ref{eq_0} and feed into classification layer to get the probability of each relation
\begin{equation}\label{eq_prob_same_space}
    P(r \mid h, t, s) = softmax(\mlp([r_{v_i,v_j} \nompar \vec{e}_h^{~''} \nompar \vec{e}_t^{~''}] ))
\end{equation}
where $\vec{e}_h^{~''}$ and $\vec{e}_t^{~''}$ are the entity embeddings learned from previous module in equation \ref{eq_gat_2}. \\
\textbf{Aggregating the separate space embeddings}: The probability in equation \ref{eq_prob_same_space} uses the embeddings learned in the same vector space. For the embeddings learned in separate vector spaces, we compute the similarity of the logits with the corresponding relation vector i.e. we use the embedding learned in equation \ref{eq_2} to find the probability of a triple exhibiting a valid relation. 
For the same, we concatenate the entity embeddings from equation \ref{eq_2} with the Equation \ref{eq_0}. This is then transformed as below:
\begin{equation}\label{eq_h_final}
 v_{htr} =  \sigma\left(W[r_{e_h,e_t} \nompar \vec{e}_h^{~r} \nompar \vec{e}_t^{~r}]\right)
\end{equation}
 Where $v_{htr}$ is a vector obtained by applying a non-linear function $\sigma$ on the final representation in the aggregator. We then compute the distance between this embedding and the relation vector $\vec{r}$ (aka $\vec{r}_k^{~''}$) obtained in the equation \ref{eq_gat_r} to get the probability of the relation existing between the two entities.
\begin{equation}\label{eq_sep_space_sigmoid}
\small
    P(r \mid h,t,\mathcal{W},A,G) = \sigma\left(\vec{r}^{\top}v_{htr}\right)
\end{equation}
where $h,t$ are the head and tail entities, $\mathcal{W}$ is the sentence, $A$ is the context, and $G$ is the computed graph.
Optimizing equation \ref{eq_sep_space_sigmoid} using binary cross entropy loss with negative sampling on the invalid triples is computationally expensive. Hence, we take the translation of this entity pair and compare with every relation. Specifically, obtain the norm of the distance metric as in \ref{eq_d} and concatenate these norms for every relation to get a translation vector.
\begin{equation}\label{eq_htr}
\small
    d_{\uptau_{htr}} = \vec{e}_h^{~r}+\vec{r}_k^{~''}-\vec{e}_t^{~r}
\end{equation}

\begin{equation}\label{eq_sep_space_attn}
\small
    \vec{t_{ht}} = \nompar_{i=1}^{N_r}\left\lVert d_{\uptau_{htr}} \right\rVert
\end{equation}
$\vec{t_{ht}}$ is translation vector of the entity pair $e_h$ and $e_t$ which represents the closeness with each relation $r$ and $N_r$ is the number of relations in the KG. This is concatenated with the vectors learned from the propagation stage and entity embeddings to classify the target relation.
\begin{multline}\label{eq_prob_sep_space_attn}
\small
    P(r \mid h, t, s) =  softmax(\mlp([r_{v_i,v_j}\parallel\vec{e}_h\parallel\vec{e}_t\parallel\vec{t_{ht}}))
\end{multline}    

\section{Experimental Setup}\label{sec:experiment}
\subsection{Datasets} 
We use two standard datasets for our experiment. (i) \textit{Wikidata dataset} \cite{DBLP:conf/emnlp/SorokinG17} created in a distantly-supervised manner by linking the Wikipedia English Corpus to Wikidata and includes sentences with multiple relations. 
It has 353 unique relations, 372,059 sentences in training, and 360,334 for testing.
(ii) \textit{NYT Freebase dataset} which was annotated by linking New York Times articles with Freebase KG \cite{DBLP:conf/pkdd/RiedelYM10}. This dataset has 53 relations (including no relation “NA”). The number of sentences in training and test set are 455,771 and 172,448 respectively.
We augment both datasets with our proposed context. For EAC, we used dumps of Wikidata\footnote{\url{https://dumps.wikimedia.org/wikidatawiki/entities/}} and Freebase\footnote{\url{https://developers.google.com/freebase}} to retrieve entity properties. In addition, the 1\&2 hop triples are retrieved from the local KG associated with each dataset.

\subsection{RECON Configurations}
We configure RECON model applying various contextual input vectors detailed below:\\
\textbf{KGGAT-SEP:} this implementation encompasses only KGGAT-SEP module of RECON (cf. section \ref{sec:triple}) which learns triple context. This is for comparing against \cite{DBLP:conf/acl/NathaniCSK19}. \\
\textbf{RECON-EAC:} induces encoded entity attribute context (from section \ref{sec:eac}) along with sentence embeddings into the propagation layer of context aggregator module. \\
\textbf{RECON-EAC-KGGAT:} along with sentence embeddings, it consumes both types of context i.e., entity context and triple context (cf. section \ref{sec:triple}) where relation and entity embeddings from the triples are trained on same vector space. \\
\textbf{RECON:} similar to RECON-EAC-KGGAT, except entity and relation embeddings for triple context learner are trained in different vector spaces.

\begin{figure*}
  \begin{subfigure}[b]{0.41\textwidth}
    \includegraphics[width=\textwidth]{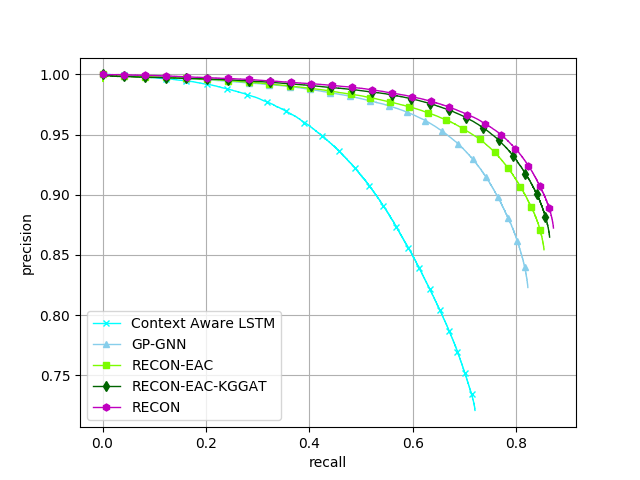}
    \caption{Micro P-R Curve}
    \label{fig:1}
  \end{subfigure}
  \begin{subfigure}[b]{0.41\textwidth}
    \includegraphics[width=\textwidth]{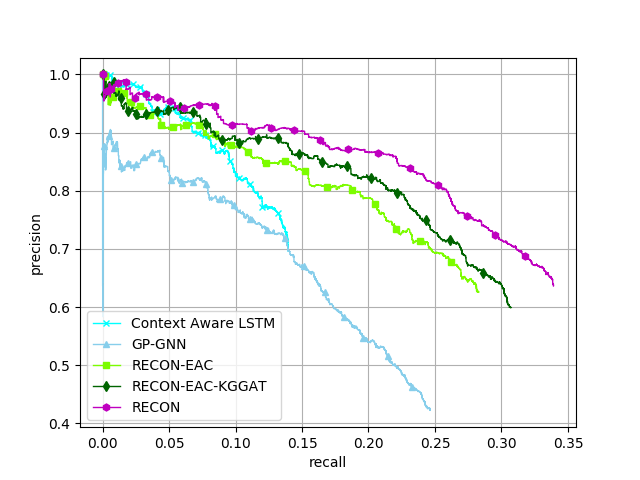}
    \caption{Macro P-R Curve}
    \label{fig:2}
  \end{subfigure}
  \caption{The P-R curves for Sentential RE approaches on Wikidata Dataset. RECON and its configurations maintain a higher precision (against the baselines) over entire recall range.}
  \label{fig:prcurve}
      \vspace{-3mm}
\end{figure*}

\begin{figure}[h!]
  \includegraphics[scale=0.51]{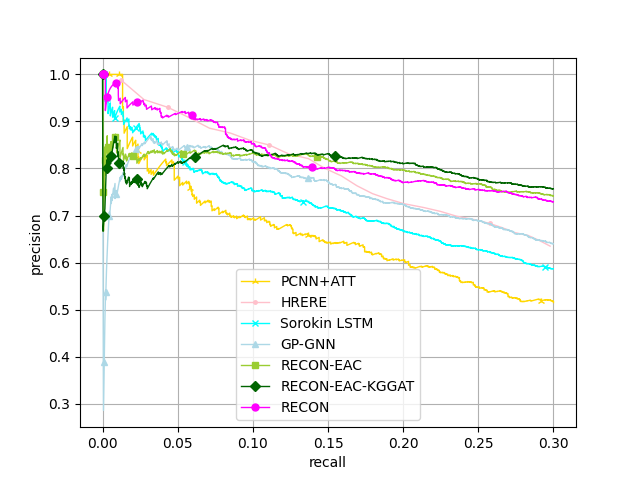}
  \caption{The P-R curves for RE approaches on NYT Freebase Dataset. We observe similar behavior as Figure \ref{fig:prcurve}, where RECON and its configurations consistently maintain a higher precision (against the baselines) over entire recall range.}
  \label{fig:nyt_micro}
\end{figure}

\subsection{Comparative Models}
We consider the recent state-of-the-art approaches for our comparison study as follows:\\
\textbf{KBGAT} \cite{DBLP:conf/acl/NathaniCSK19}: this open-source implementation is compared with our KGGAT-SEP for evaluating the effectiveness of our approach in learning the KG triple context.\\
\textbf{Context-Aware LSTM} \cite{DBLP:conf/emnlp/SorokinG17}: learns context from other relations in the sentence. We reuse its open-source code.\\
\textbf{GP-GNN} \cite{DBLP:conf/acl/ZhuLLFCS19}: proposes multi-hop reasoning between the entity nodes for sentential RE. We employ the open source code.\\
\textbf{Sorokin-LSTM} \cite{DBLP:conf/emnlp/SorokinG17}: NYT Freebase dataset contains one relation per sentence, but Context-Aware LSTM has a prerequisite of having at least two relations in a sentence. Hence, we reuse another baseline, which is an LSTM model without a sentential relation context.\\
\textbf{Multi-instance RE approaches}:
consider context from the surrounding text of a given sentence whereas Sentential RE limits context only to the given sentence.
Our idea is to observe if inducing KG context into a sentential RE model can be a good trade-off against a multi-instance setting.
Hence, we compare RECON and other sentential RE baselines (Sorokin-LSTM \& GP-GNN) with the  multi-instance RE models. For this, we rely on the NYT Freebase dataset, since the other dataset does not have multiple instances for an entity pair. \textbf{HRERE} \cite{DBLP:conf/naacl/XuB19} is the multi-instance SOTA on NYT Freebase dataset that jointly learns different representations from text and KG facts. For the completion of the comparison, performance of four previous baselines are also reported i.e., (i) \textbf{Wu-2019} \cite{DBLP:conf/aaai/WuFZ19}, (ii) \textbf{Yi-Ling-2019} \cite{DBLP:conf/naacl/YeL19}, (iii) \textbf{RESIDE} \cite{DBLP:conf/emnlp/VashishthJPBT18}, and iv) \textbf{PCNN+ATTN} \cite{lin2016neural}. The values are taken from the respective papers.



\subsection{Hyperparameters and Metric}
The EAC module (section \ref{sec:eac}) uses a biLSTM with one hidden layer of size 50. The convolution filter is of width one, and the output size is 8. In KGGAT-SEP (section \ref{sec:triple}), the initial entity and relation embedding size is 50, number of heads are two with two GAT layers, and the final entity and relation embedding size is 200. For the context aggregator module, we adapt the parameters provided in GP-GNN \cite{DBLP:conf/acl/ZhuLLFCS19}. The word embedding dimension is 50 initialized from the Glove embeddings. The position embedding is also kept at 50 dimensions. Encoder uses a layer of bidirectional LSTM of size 256. We use three propagation layers with the entity embedding dimension set at 8. For brevity, complete training details and validation results are in the public Github. \\
\textbf{Metric and Optimization:}  Similar to baseline, we ignore probability predicted for the NA relation during testing on both datasets. We use different metrics depending on the dataset as per the respective baselines for fair comparison. On Wikidata dataset, we adapt (micro and macro) precision (P), recall (R), and F-score (F1) from \cite{DBLP:conf/emnlp/SorokinG17}. For NYT Freebase dataset, we follow the work by \cite{DBLP:conf/naacl/XuB19} that uses (micro) P@10 and P@30. 
An ablation is performed to measure effectiveness of KGGAT-SEP in learning entity and relation embeddings. For this, we use the hits@N, average rank, and average reciprocal rank in similar to \cite{DBLP:conf/acl/NathaniCSK19}. Our work employs the Adam optimizer \cite{DBLP:journals/corr/KingmaB14} with categorical cross entropy loss where each model is run three times on the whole training set. For the P/R curves, we select the results from the first run of each model. Our experiment settings are borrowed from the baselines: GP-GNN \cite{DBLP:conf/acl/ZhuLLFCS19} for the sentential RE and HRERE~\cite{DBLP:conf/naacl/XuB19} for the multi-instance RE.



\begin{table}[ht!]
 \small
    \centering
    \begin{tabular}{p{3cm}|p{0.5cm}p{0.5cm}p{0.5cm}|p{0.5cm}p{0.5cm}p{0.5cm}}
   \toprule
    &\multicolumn{3}{c|}{\textbf{Micro}}& \multicolumn{3}{c}{\textbf{Macro}}\\
     \textbf{Model} &P & R & F1  &P & R & F1 \\
    \midrule
    Context-Aware LSTM \cite{DBLP:conf/emnlp/SorokinG17} & 72.09 &72.06 &72.07 &\textbf{69.21} &13.96 &17.20 \\
     GP-GNN \cite{DBLP:conf/acl/ZhuLLFCS19}& 82.30  &82.28  &82.29 & 42.24  & 24.63  & 31.12\\
   \midrule
    RECON-EAC  & 85.44 & 85.41 & 85.42 & 62.56 & 28.29 & 38.96\\
    RECON-EAC-KGGAT & 86.48 &  86.49 & 86.48 & 59.92 & 30.70 & 40.60
    \\
     RECON & \textbf{87.24} &  \textbf{87.23} & \textbf{87.23} & 63.59 & \textbf{33.91} & \textbf{44.23}
    \\
    \bottomrule
 \end{tabular}
\caption{Comparison of RECON and sentential RE models on the Wikidata dataset. Best values are in bold. Each time a KG context is added in a graph neural network, the performance has increased, resulting in a significant RECON outperformance against all sentential RE baselines.}
\label{tab:results1}
   \vspace{-2mm}
\end{table}

\section{Results} \label{sec:evaluation}
We study following research questions: "\textbf{RQ1}: How effective is RECON in capturing the KG context-induced in a graph neural network for the sentential RE?" The research question is further divided into two sub-research questions: \textbf{RQ1.1}: what is the useful contribution of each entity attribute context (alias, instance-of, type, label in RECON-EAC) for sentential RE? \textbf{RQ1.2}: How effective is separation of entity and relation embedding spaces (RECON-KGGAT-SEP) in capturing the triple context from neighborhood 1\&2 hop triples for the given entities? \textbf{RQ2}: Is the addition of the KG context statistically significant? Each of our experiments systematically studies the research questions in different settings.\\
\textbf{Performance on Wikidata dataset}: Table \ref{tab:results1} summarizes the performance of RECON and its configurations against other sentential RE models. It can be observed that by adding the entity attribute context (RECON-EAC), we surpass the baseline results. The RECON-EAC-KGGAT values indicate that when we further add context from KG triples, there is an improvement. 
However, the final configuration RECON achieves the best results. It validates our hypothesis that RECON is able to capture the KG context effectively. 
The P/R curves are illustrated in the Figure \ref{fig:prcurve}.
RECON steadily achieves higher precision over the entire recall range compared to other models.
In running example (cf. Figure \ref{fig:approach}), RECON could predict the correct relation \textit{wdt:P26 (spouse)} between \textit{wdt:Q76 (Barack Obama)} and \textit{wdt:Q13133 (Michelle Obama)}, while, the other two baselines wrongly predicted the relation \textit{wdt:P155 (follows)}. 

\textbf{Performance on NYT Freebase Dataset:} RECON and its configurations outperforms the sentential RE baselines (cf. Table \ref{tab:results2}). Hence, independent of underlying KG, RECON can still capture sufficient context collectively from entity attributes and factual triples. We also compare the performance of sentential RE models, including RECON and its configurations against multi-instance RE baselines. It can be deducted from Table \ref{tab:results2} that RECON supersedes the performance of multi-instance baselines. Furthermore, the RECON's P/R curve for the NYT Freebase dataset shown in Figure \ref{fig:nyt_micro} maintains a higher precision over the entire recall range. The observation can be interpreted as follows: adding context from the knowledge graphs instead of the bag of sentences for the entity pairs keeps the precision higher over a more extended recall range. Hence, we conclude that RECON is effectively capturing the KG context across KGs, thereby answering the first research question \textbf{RQ1} successfully.


\begin{table}[ht!]
\small
    \centering
    \begin{tabular}{p{1.3cm}|p{3cm}|p{0.75cm}|p{0.75cm}}
     \toprule
     & &\multicolumn{2}{c}{\textbf{Precision}}   \\
     \textbf{Task} & \textbf{Model} & @10\% & @30\% \\
     \midrule
     \multirow{4}{*}{Sentential} & Sorokin-LSTM \cite{DBLP:conf/emnlp/SorokinG17}  & 75.4 & 58.7\\
     & GP-GNN \cite{DBLP:conf/acl/ZhuLLFCS19}  & 81.3 & 63.1\\
      \cline{2-4}
     & RECON-EAC  &  83.5& 73.4 \\
     & RECON-EAC-KGGAT  &  86.2 & 72.1 \\
     & RECON & \textbf{87.5} &  \textbf{74.1} \\
     \cline{1-4}
     \multirow{4}{3pt}{Multi-instance} & HRERE~\cite{DBLP:conf/naacl/XuB19} & 84.9 & 72.8 \\
     & Wu-2019 \cite{DBLP:conf/aaai/WuFZ19}  & 81.7 & 61.8\\
     & Ye-Ling-2019 \cite{DBLP:conf/naacl/YeL19}& 78.9 & 62.4 \\
     & RESIDE \cite{DBLP:conf/emnlp/VashishthJPBT18} & 73.6 &  59.5 \\
     & PCNN+ATTN \cite{lin2016neural} & 69.4 &  51.8 \\
    \bottomrule
     
    \end{tabular}
\caption{Comparison of RECON against baselines (sentential and multi-instance) on the NYT Freebase dataset. Best values are in bold. RECON continues to significantly outperform sentential RE baselines and also surpasses the performance of state of the art multi-instance RE approach.}
\label{tab:results2}
    \vspace{-2mm}
\end{table}

\begin{table}[!htb]
\small
    \centering
    \begin{tabular}{@{}p{2.6cm}|p{0.6cm}| p{0.6cm}|@{}p{1cm}|@{}p{1.2cm}|@{}p{0.75cm}}
        \Cline{1-6}
        \textbf{Compared Models} & \multicolumn{2}{|c|}{\textbf{Contingency}} & \textbf{Statistic} & \textbf{p-value} & \textbf{Dataset} \\
        \hline
         GP-GNN Vs & 568469 & 40882 &4978.84  &0.0 &Wikidata  \\
          RECON-EAC & 63702 & 67713 &  & &\\
          \hline
         RECON-EAC Vs &599135 &33036 &862.38  &$1.5*10^{-189}$ &Wikidata  \\
          RECON-EAC-KGGAT &41029 &67566 &  & &  \\
          \hline
         RECON-EAC-KGGAT &608442 &31722 &455.29  &$5.1*10^{-101}$ &Wikidata  \\
         Vs RECON &37330 &63272 &  & & \\
           \Cline{1-6}
         GP-GNN Vs &158426 &4936 &15.72  &$7.3*10^{-5}$ &Freebase   \\
          RECON-EAC &53392 &3699 &  & &\\
          \hline
           RECON-EAC Vs &160227 &3538 &59.44  &$1.2*10^{-14}$  &Freebase  \\
          RECON-EAC-KGGAT &4218 &4417 & &  \\
          \hline
          
         RECON-EAC-KGGAT &161012 &3433 &54.88  &$1.3*10^{-13}$ &Freebase \\
         Vs RECON &4076 &3879 &  & &\\
        \Cline{1-6}
    \end{tabular}
    \caption{The McNemar's test for statistical significance on the results of both datasets. It can be observed that each of the improvement in the RECON configurations is statistically significant independent of the underlying KG. }
    \label{mcnemars_test}
    \vspace{-3mm}
\end{table}

\subsection{Ablation Studies}
\textbf{Effectiveness of EAC}: We separately studied each entity attribute's effect on the performance of the RECON-EAC. Table \ref{tab:tab-abl1} and Table \ref{tab:tab-abl1nyt} summarize the contribution of the four entity attributes when independently added to the model. The entity type (Instance-of) contributes the least across both datasets. We see that the entity descriptions significantly impact RECON's performance on the Wikidata dataset, while descriptions have not provided much gain on Freebase. The Freebase entity descriptions are the first paragraph from the Wikipedia entity web page, whereas, for Wikidata, descriptions are a human-curated concise form of the text. Mulang' et al. \cite{mulang2020evaluating} also observed that when the Wikipedia descriptions are replaced with the entity descriptions derived from the Wikidata KG, the performance of an entity disambiguation model increases.

The reported study on the EAC module's effectiveness answers our first sub-research question (\textbf{RQ1.1}). We conclude that the contribution of entity attributes in the EAC context varies per underlying KG. Nevertheless, once we induce cumulative context from all entity attributes, we attain a significant jump in the RECON-EAC performance (cf. Table \ref{tab:results1} and Table \ref{tab:results2}).

\begin{table}[!htb]
    \centering
    \begin{tabular}{p{4.5cm}|p{0.8cm}|p{0.8cm}|p{0.8cm}}
        \toprule
        \textbf{Model} & \textbf{P} & \textbf{R} & \textbf{F1} \\
        \midrule
         RECON-EAC(Instance of) & 76.33 & 76.32 & 76.32 \\
        RECON-EAC(label) & 78.64 & 78.70 & 78.67 \\
         RECON-EAC(Alias) & 81.58 & 81.56 & 81.57 \\
          RECON-EAC(Description) & \textbf{83.16} & \textbf{83.18} & \textbf{83.17} \\
        \bottomrule
    \end{tabular}
    \caption{RECON-EAC performance on Wikidata Dataset. The rows comprise of the configuration when context from each entity attribute is added in isolation. We report micro P, R, and F scores. (Best score in bold)}
    \label{tab:tab-abl1}
        \vspace{-2mm}
\end{table} 

\begin{table}[!htb]
    \centering
    \begin{tabular}{p{3.5cm}|p{0.8cm}|p{0.8cm}}
        \toprule
        \textbf{Model} & \textbf{P@10} & \textbf{P@30} \\
        \midrule
         RECON-EAC(Instance of) & 71.83 & 57.52  \\
        RECON-EAC(label) & 78.14 & 66.34 \\
         RECON-EAC(Alias) & \textbf{80.60} & \textbf{67.13} \\
          RECON-EAC(Description) & 72.40 & 67.11 \\
        \bottomrule
    \end{tabular}
    \caption{RECON-EAC performance on NYT Freebase Dataset. The rows comprise of the configuration when context from each entity attribute is added in isolation. We report P@10 and P@30, similar to other NYT dataset experiments. (Best score in bold)}
    \label{tab:tab-abl1nyt}
        \vspace{-2mm}
\end{table}

\begin{table*}[!htbp]
\small
    \centering
    \begin{tabular}{p{8cm}|p{2.8cm}|p{1.2cm}|p{1.2cm}|p{1.2cm}|p{1.2cm}}
        \Cline{1-6}
         & & &\textbf{Context-} & & \\
        \textbf{Sentence} & \textbf{Entities} & \textbf{Correct} & \textbf{ Aware}  & \textbf{GP-GNN}\cite{DBLP:conf/acl/ZhuLLFCS19} & \textbf{RECON} \\
        & &\textbf{Relation} &\textbf{LSTM}\cite{DBLP:conf/emnlp/SorokinG17} & &  \\
        \hline
         \multirow{2}{8cm}{1. Specifically , the rapper listed Suzanne Vega , Led Zeppelin , Talking Heads , Eminem , and Spice Girls. } & Q5608  : Eminem & P106 & P31 & P31 & P106 \\
        & Q2252262 : rapper & Occupation& Instance of & Instance Of & Occupation \\
       \hline
        \multirow{2}{8cm}{2. Bocelli also took part in the Christmas in Washington special on Dec 12, in the presence of president Barack Obama and the first lady}  & Q76 : Barack Obama    & P26  & P155 & P155  & P26 \\
        &  Q13133 : Michelle Obama  & spouse & follows& follows & spouse \\
         &  & &  &  &\\
        
        \hline
         \multirow{2}{8cm} {3. It was kept from number one by Queen's Bohemian Rhapsody } & Q15862 : Queen & P175  & P50 & P50  & P175 \\
        &  Q187745 : Bohemian Rhapsody& performer & author & author & performer \\
         &  \hspace{1.35cm} & &  & & \\

        
        
        \Cline{1-6}
    \end{tabular}
    \caption{Sample sentence examples from the Wikidata dataset. RECON is able to predict the relations which are not explicitly observable from the sentence itself.}
    
    \label{tab:tab1}
\end{table*}

\begin{table}[!htb]
    \centering
    \begin{tabular}{p{1.8cm}|p{1.5cm}|p{0.8cm}|p{0.8cm}|p{2.0cm}}
       \toprule
        \textbf{Model} & \%\textbf{Hits @10} & \textbf{MR} & \textbf{MRR}& \textbf{Dataset} \\
       \midrule
        KBGAT & 65.8 & 35.2 & 0.36 & Wikidata\\
        KGGAT-SEP  & \textbf{72.6} &\textbf{29} & \textbf{0.38}& Wikidata \\
        \Cline{1-5}
         KBGAT & 85.8 & 7.48 & 21.6 & NYT Freebase\\
        KGGAT-SEP  & \textbf{88.4} &\textbf{5.42} & \textbf{32.3}& NYT Freebase \\
      \bottomrule
    \end{tabular}
    \caption{Comparing KGGAT-SEP and KBGAT for triple classification task on both Datasets. We conclude that separating the entity and relation embedding space has been beneficial for the triple classification task, hence, contributing positively to RECON performance. (cf. Table \ref{tab:results1} and \ref{tab:results2}).}
    \label{tab:tab-abl2}
        \vspace{-2mm}
\end{table}

\textbf{Understanding the KG triple Context}: To understand the effect of relying on one single embedding space or two separate spaces, we conducted an ablation study for the triple classification task. We performed a ranking of all the triples for a given entity pair and obtained hits@N, average rank, and Mean Reciprocal Rank (MRR). Hits@10 denotes the fraction of the actual triples that are returned in the top 10 predicted triples. Table \ref{tab:tab-abl2} illustrates that the KGGAT-SEP (separate spaces) exceeds KBGAT (single space) by a large margin on the triple classification task. Training in separate vector spaces facilitates learning more expressive embeddings of the entities and relations in the triple classification task. The positive results validate the effectiveness of the KGGAT-SEP module and answers the research question \textbf{RQ1.2}. However, when we trained entity and relation embeddings of KG triples in separate spaces, improvements are marginal for the sentential RE task (cf. Table \ref{tab:results1}). We could interpret this behavior as : the model may have already learned relevant information from the sentence and the triple context before we separate vector spaces. Also, in our case, the computed graph is sparse for sentential RE, i.e., few relations per entity prevents effective learning of good representation \cite{DBLP:conf/acl/NathaniCSK19}. We believe sparseness of the computed graph has hindered effective learning of the entity embeddings. It requires further investigation, and we plan it for our future work.

\textbf{Statistical Significance of RECON}: The McNemar's test for statistical significance has been used to find if the reduction in error at each of the incremental stages in RECON are significant. The test is primarily used to compare two supervised classification models \cite{dietterich1998approximate}. The results are shown in Table \ref{mcnemars_test}.
For the column "contingency table" (2x2 contingency table), the values of the first row and second column ($RW$) represent the number of instances that model 1 predicted correctly and model 2 incorrectly. Similarly the values of the second row and first column gives the number of instances that model 2 predicted correctly and model 1 predicted incorrectly ($WR$). The statistic here is 
\[\frac{(RW-WR)^2}{RW+WR}\]
The differences in the models are said to be statistically significant if the $p-value<0.05$ \cite{dietterich1998approximate}. On both datasets, for all RECON configurations, the results are statistically significant, illustrating our approach's robustness (also answering second research question \textbf{RQ2}). In the contingency table, the ($RW$) values provide an exciting insight. For example, in the first row of the Table \ref{mcnemars_test}, there are 40882 sentences for which adding the RECON-EAC context has negatively resulted in the performance compared to GP-GNN. This opens up a new research question that how can one intelligently select the KG context based on the sentence before feeding it into the model. We leave the detailed exploration for future work. \\
\textbf{Performance on Human Annotation Dataset}: To provide a comprehensive ablation study, \cite{DBLP:conf/acl/ZhuLLFCS19} provided a human evaluation setting and reports Micro P, R, and F1 values. Following the same setting, we asked five annotators\footnote{Annotators were well-educated university students.} to annotate randomly selected sentences from Wikidata dataset \cite{DBLP:conf/emnlp/SorokinG17}. The task was to see whether a distantly supervised dataset is right for every pair of entities. Sentences accepted by all annotators are part of the human-annotated dataset. There are 500 sentences in this test set. Table \ref{tab:human} reports RECON performance against the sentential baselines. We could see that RECON and its configurations continue to outperform other sentential RE baselines. The results further re-assure the robustness of our proposed approach. 

\begin{figure}[h!]
  \includegraphics[scale=0.51]{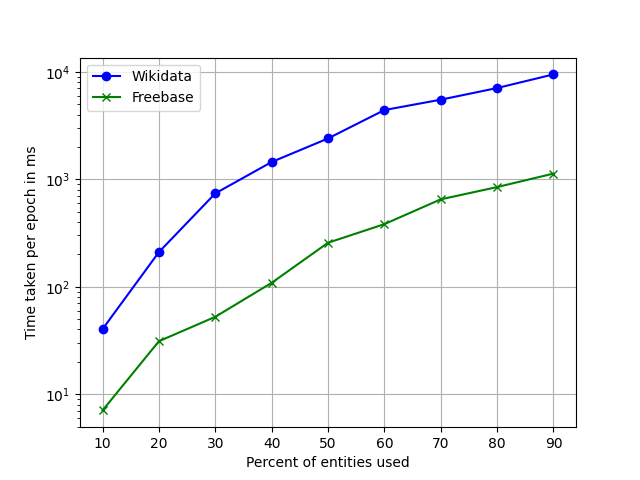}
  \caption{Scalability of Triple Context Learner (KGGAT-SEP) on Wikidata and NYT Freebase datasets. When we incrementally added entity nodes in the KB to capture the triple context, the training time increases by a polynomial factor.} 
  \label{fig:scale}
\end{figure}

\begin{table*}[!htb]
    \centering
    \begin{tabular}{p{2.5cm}|p{0.8cm}p{0.8cm}|p{0.8cm}p{0.8cm}|p{1.0cm}p{1.0cm}|p{1.2cm}p{1.2cm}|p{0.8cm}p{0.8cm}}
        \toprule
        Relation type & \multicolumn{2}{c|}{Context Aware}& \multicolumn{2}{c|}{GPGNN} & \multicolumn{2}{c|}{RECON-EAC} & \multicolumn{2}{c|}{RECON-EAC-KGGAT} & \multicolumn{2}{c}{RECON} \\
         & P & R & P & R& P & R & P & R & P & R \\
        \midrule
        COUNTRY & 0.449 & 0.435 & 0.873& 0.931& 0.952& 0.951& 0.953& \textbf{0.956}& \textbf{0.955}& 0.949  \\
        LOCATED IN & 0.476& 0.130& 0.463& 0.074& 0.462& 0.355& \textbf{0.556}& 0.300& 0.398& \textbf{0.481}  \\
        SHARES BORDER & 0.724& 0.725& 0.732& 0.817& 0.838& 0.843& 0.839& \textbf{0.868}& \textbf{0.862}& 0.829   \\ 
        INSTANCE OF & 0.748& 0.745& 0.780& 0.748& \textbf{0.916}& 0.912& 0.857& \textbf{0.938}& 0.896& 0.916   \\
        SPORT & 0.966& 0.968& 0.962& 0.975& 0.988& 0.991& \textbf{0.989}& 0.990& 0.987& \textbf{0.991}  \\
        CITIZENSHIP & 0.853& 0.895& 0.848& 0.913& 0.963& 0.968& \textbf{0.964}& \textbf{0.971}& 0.962& 0.966  \\
        PART OF & 0.462& 0.427& 0.443& 0.441& 0.558& 0.662& 0.596& \textbf{0.622}& \textbf{0.628}& 0.565  \\
        SUBCLASS OF & 0.469& 0.375& 0.435& 0.498& 0.645& 0.709& \textbf{0.640}& 0.619& 0.588& \textbf{0.772}  \\
        \bottomrule
    \end{tabular}
    \caption{Precision and Recall of the top relations (as per number of occurrences) in the Wikidata dataset. Induction of KG context in RECON and configurations demonstrate the most improvement on precision across all relation categories.}
    \label{tab:tab-toprelation}
        \vspace{-2mm}
\end{table*} 

\subsubsection{Case Studies}
We conducted three case studies. For the first case study, Table \ref{tab:tab1} demonstrates RECON's performance against two sentential baselines: Context-Aware LSTM \cite{DBLP:conf/emnlp/SorokinG17} and GP-GNN \cite{DBLP:conf/acl/ZhuLLFCS19} on a few randomly selected sentences from the Wikidata dataset. We can see that these sentences don't directly contain much information regarding the potential relationship between two entities (the relations are implicitly coded in the text). For example, in the first sentence, the relation between the entities rapper and Eminem is "occupation." The baselines predicted "Instance of" as the target relation considering sentential context is limited. However, the Wikidata description of the entity Q8589(Eminem) is "American rapper, producer and actor". Once we feed description in our model as context for this sentence, RECON predicts the correct relation. 

\begin{table}[!htb]
    \centering
   \begin{tabular}{p{4.5cm}|p{0.8cm}|p{0.8cm}|p{0.8cm}}
       \toprule
        & & & \\
      \textbf{Model} & \textbf{P} & \textbf{R} & \textbf{F1} \\
        \Cline{1-4}
         Context Aware LSTM \cite{DBLP:conf/emnlp/SorokinG17} & 77.77 & 78.69 & 78.23 \\
       GP-GNN \cite{DBLP:conf/acl/ZhuLLFCS19} & 81.99 & 82.31 & 82.15 \\
        \hline
         RECON-EAC & 86.10 & 86.58 & 86.33 \\
          RECON-KBGAT& 86.93 & 87.16 & 87.04 \\
           RECON& \textbf{87.34} & \textbf{87.55} & \textbf{87.44} \\
        \bottomrule
    \end{tabular}
    \caption{Sentential RE performance on Human Annotation Dataset. RECON again outperforms the baselines. We report Micro P,R, and F1 values. (Best score in bold)}
    \label{tab:human}
        \vspace{-2mm}
\end{table}

Sorokin et al. \cite{DBLP:conf/emnlp/SorokinG17} provided a study to analyze the impact of their approach on top relations (acc. to the number of occurrences) in Wikidata dataset. Hence, in the second case study, we compare the performance of RECON against sentential RE baselines for the top relations in Wikidata dataset (cf. Table \ref{tab:tab-toprelation}). We conclude that the KG context has positively impacted all top relation categories and appears to be especially useful for taxonomy relations (INSTANCE OF, SUBCLASS OF, PART OF). 

The third case study focuses on the scalability of Triple Context Learner (KGGAT-SEP) on both datasets. We incrementally add a fraction of entity nodes in the KB to capture the neighboring triples' context. Our idea here is to study how training times scale with the size of the considered KB. Figure \ref{fig:scale} illustrates that when we systematically add entity nodes in the KB, the time increases by a polynomial factor, which is expected since we consider the 2 hop neighborhood of the nodes.

\section{Conclusion and Future Directions} \label{sec:conclusion}
This paper presents RECON, a sentential RE approach that integrates sufficient context from a background KG. Our empirical study shows that KG context provides valuable additional signals when the context of the RE task is limited to a single sentence. Gleaning from our evaluations, we conclude three significant findings: i) the simplest form of KG context like entity description already provide ample signals to improve the performance of GNNs. We also see that proper encoding of combined entity attributes (labels, descriptions, instance of, and aliases) results in more impacting knowledge representation. ii) Although graph attention networks provide one of the best avenue to encode KG triples, more expressive embeddings can be achieved when entity and relation embeddings are learned in separate vector spaces. iii) Finally, due to the proposed KG context and encoding thereof, RECON transcends the SOTA in sentential RE while also achieving SOTA results against multi-instance RE models. The Multi-instance setting, which adds context from the previous sentences of the bag is a widely used practice in the research community since 2012 \cite{DBLP:conf/emnlp/SurdeanuTNM12,DBLP:conf/naacl/XuB19,DBLP:conf/aaai/WuFZ19}. We submit that sentential RE models induced with effectively learned KG context could be a good trade-off compared to the multi-instance setting. We expect the research community to look deeper into this potential trade-off for relation extraction.

Based on our findings, exhaustive evaluations, and gained insights in this paper, we point readers with the following future research directions: 1) Results reported in Table \ref{mcnemars_test} illustrate that there exist several sentences for which KG context offered minimal or negative impact. Hence, it remains an open question of how an approach intelligently selects a specific form of the context based on the input sentence. 2) We suggest further investigation on optimizing the training of embeddings in separate vector spaces for RE. We also found that combining the triple context with the entity attribute context offered minimal gain to the model. Hence, we recommend jointly training the entity attribute and triple context as a viable path for future work. 3) The applicability of RECON in an industrial scale setting was out of the paper's scope. The researchers with access to the industrial research ecosystem can study how RECON and other sentential RE baselines can be applied to industrial applications. 4) The data quality of the derived KG context directly impacts the performance of knowledge-intense information extraction methods \cite{weichselbraun2018mining}. The effect of data quality of the KG context on RECON is not studied in this paper's scope and is a viable next step.

\section*{Acknowledgment} We thank Satish Suggala for additional server access and anonymous reviewers for very constructive reviews.

\bibliographystyle{ACM-Reference-Format}
\bibliography{bibliography}
\section{Appendix}
\subsection{Theoretical Motivation} \label{sec:theory}
We define a set of theorems that motivated our approach RECON and provided a theoretical foundation.
\begin{lemma}\label{lemma1}
If entity and relation embeddings are expressed in the same vector space, there cannot be more than one distinct relation per entity pair
\end{lemma}
\begin{proof}
Consider two entities $\vec{e}_1$ and $\vec{e}_2$. Consider a relation $\vec{r}_1$ between them. We want to have these vectors satisfy the triangle law of vector addition as below
\begin{equation}\label{lemma1_tl1}
\vec{e}_1 + \vec{r}_1 = \vec{e}_2    
\end{equation}
Now assume another relation $\vec{r}_2$ between $\vec{e}_1$ and $\vec{e}_2$ (where $\vec{e}_1$ is the subject). Thus we have,
\begin{equation}\label{lemma1_tl2}
\vec{e}_1 + \vec{r}_2 = \vec{e}_2    
\end{equation}
From lemmas \ref{lemma1_tl1} and \ref{lemma1_tl2} we get: $\vec{r}_1 = \vec{r}_2$
\end{proof}

\begin{lemma}\label{lemma2}
If entity and relation embeddings are expressed in the same vector space, there can not exist a single common relation between an entity and two different, directly connected entities
\end{lemma}
\begin{proof}
Consider $\vec{e}_1$ and $\vec{e}_2$ to have relation $\vec{r}_1$. Consider $\vec{e}_1$ and $\vec{e}_3$ to have the same relation $\vec{r}_1$. Then,

\begin{equation} \label{eq_a}
\begin{aligned}
\vec{e}_1 + \vec{r}_1 = \vec{e}_2;   \vec{e}_1 + \vec{r}_1 = \vec{e}_3; \\
\therefore \vec{e}_2 - \vec{e}_3 = \vec{0}; \vec{e}_2 = \vec{e}_3 
\end{aligned}
\end{equation}

\noindent We call this problem a mode collapse as the two separate entity embeddings collapse into a single vector.
\end{proof}

\begin{lemma}\label{lemma3}
If entity and relation embeddings are expressed in the same vector space, no entity is sharing a common relation between two indirectly related entities
\end{lemma}
\begin{proof}
Consider $\vec{e}_1$ and $\vec{e}_2$ to have a relation $\vec{r}_1$. Consider $\vec{e}_1$ and $\vec{e}_3$ to have a relation $\vec{r}_3$. Let $\vec{r}_1$ and $\vec{r}_3$ be inverse relations. Assume $\vec{r}_1, \vec{r}_2 \neq 0$

\begin{equation} \label{eq_b}
\begin{aligned}
\vec{r}_1=-\vec{r}_2;\vec{r}_1=-\vec{r}_2 \\
\vec{e}_1+\vec{r}_2=\vec{e}_3; \vec{e}_2-\vec{e}_3=2\vec{r}_1
\end{aligned}
\end{equation}
Now consider $\vec{e}_4$ to have a common relation with $\vec{e}_2$ and $\vec{e}_3$. Let this relation be $\vec{r}_3$.

\begin{equation} \label{eq_c}
\begin{aligned}
\vec{e}_2+\vec{r}_3=\vec{e}_4 ; \vec{e}_3+\vec{r}_3=\vec{e}_4\\
\vec{e}_2-\vec{e}_3=\vec{0}; \vec{r}_1=\vec{0}
\end{aligned}
\end{equation}
Which contradicts the assumption
\end{proof}

\begin{lemma}\label{lemma4}
If $f_r$ is an invertible and distributive function/transform for a relation $\vec{r}$, then for an entity sharing a common relation between two other distinct entities, this function causes the embeddings of the two entities to be merged into one
\end{lemma}
\begin{proof}
Let's assume a transformation function $f_r$ that transforms from the entity to the relation space. Assuming the triangle law holds we have,
\[f_r(\vec{e}_1)+\vec{r}_1=f_r(\vec{e}_2) \: \mbox{and}\]
\[f_r(\vec{e}_1)+\vec{r}_1=f_r(\vec{e}_3)\]
\[\therefore f_r(\vec{e}_2)-f_r(\vec{e}_1)=f_r(\vec{e}_3)-f_r(\vec{e}_1)\]
\[f_r(\vec{e}_2-\vec{e}_1)=f_r(\vec{e}_3-\vec{e}_1) \:  \mbox{...since $f_r$ is distributive}\]
\[f_r^{-1} * f_r(\vec{e}_2-\vec{e}_1)=f_r^{-1} * f_r(e3-e1) \: \]
\[\mbox{..since $f_r$ is invertible}\]
\[\vec{e}_2-\vec{e}_1=\vec{e}_3=\vec{e}_1\]
\[\vec{e}_2=\vec{e}_3\]
However we may want to have $\vec{e}_2$ separate from $\vec{e}_3$. 
\end{proof}
\noindent The affine transform as used by TransR\cite{DBLP:conf/aaai/LinLSLZ15} belongs to this class of transform. Hence we propose adding a non-linear transform.

\begin{lemma}\label{lemma5}
If $\mathcal{T}_g$ is the set of triples learned under a common transform $f_g$ and $\mathcal{T}_l$ is the set of triples learned under a transform $f_l$ which is distinct per relation then $\mathcal{T}_g \subsetneq \mathcal{T}_l$ i.e. $\mathcal{T}_g$ is a strict subset of $\mathcal{T}_l$
\end{lemma}
\begin{proof}
We prove this lemma in two parts.
First we show that $\mathcal{T}_g \subseteq \mathcal{T}_l$ then we show that $\mathcal{T}_l \nsubseteq \mathcal{T}_g$.
\par \noindent 1. The first part is straightforward as we can set $f_l = f_g$ and make $\mathcal{T}_g \subseteq \mathcal{T}_l$
\par \noindent 2. For showing the second part we consider the following system of triples
Consider relations $\vec{r}_1$ and $\vec{r}_2$ between entities $\vec{e}_1$ and $\vec{e}_2$ and $\vec{r}_1 \neq \vec{r}_2$
We define a common transform $f_g$ such that
\[f_g(\vec{e}_1)+\vec{r}_1=f_g(\vec{e}_2) \: \mbox{and}\] 
\[f_g(\vec{e}_1)+\vec{r}_2=f_g(\vec{e}_2)\]
\[\therefore \vec{r}_1=\vec{r}_2\]
	
\par \noindent For the per relation transform we can define a function $f_{r_1}$ for $r_1$ and $f_{r_2}$ for $r_2$ such that \[f_{r_1}(\vec{e}_1)+\vec{r}_1=f_{r_1}(\vec{e}_2) \: \mbox{and}\] 
\[f_{r_2}(\vec{e}_1)+\vec{r}_2=f_{r_2}(\vec{e}_2)\]
such that $\vec{r}_1 \neq \vec{r}_2$
\par \noindent Thus $\mathcal{T}_l \nsubseteq \mathcal{T}_g$, and hence the proof.
\end{proof}

\begin{lemma}\label{lemma6}
If $\mathcal{T}_{gca}$ is the set of triples that can be learned under a global context aware transform $f_{gca}$ and $\mathcal{T}_{lca}$ is the set of transforms learned under a local context aware transform then $\mathcal{T}_{lca} \subsetneq \mathcal{T}_{gca}$.
By context here, we mean the KG triples, global context refers to all the triples in the KG the current entities are a part of, and local context indicates the triple under consideration.
\end{lemma}
\begin{proof}
We proceed similar to lemma \ref{lemma5}.
\par \noindent 1. We can make $f_{gca} = f_{lca}$ by ignoring the global context and thus $\mathcal{T}_{lca} \subseteq \mathcal{T}_{gca}$
\par \noindent 2. We define a globally context aware transform as below:
\[f_{gca}(\vec{e}_1) = f_r(\vec{e}_1)\]
\[f_{gca}(\vec{e}_2)= \sum_{j \in N_r(\vec{e}_1)}\alpha_j*f_r(e_j)\]
Where $\alpha_j$ is the attention value learned for the triple $<\vec{e}_1,\vec{r},\vec{e}_j>$
\par \noindent In a simple setting we can have $\alpha_j= \frac{1}{N_r}$ and learn 
\[\vec{r}=f_{gca}(\vec{e}_2)-f_{gca}(\vec{e}_1)=f_{gca}(\vec{e}_3)-f_{gca}(\vec{e}_1)\]
With $\vec{e}_2 \neq \vec{e}_3$
\par \noindent However in a local context aware transform $f_{lca}$ we have,
\[f_{lca}(\vec{e}_1)+\vec{r}=f_{lca}(\vec{e}_2)\]
\[f_{lca}(\vec{e}_1)+\vec{r}=f_{lca}(\vec{e}_3)\]
From lemma \ref{lemma4} $\vec{e}_2 = \vec{e}_3$ and thus we can not have both $<\vec{e}_1,\vec{r},\vec{e}_2>$ and $<\vec{e}_1,\vec{r},\vec{e}_3>$ in $T_l$
\par \noindent Thus $\mathcal{T}_{gca} \nsubseteq \mathcal{T}_{lca}$ and hence the proof
\end{proof}

\begin{theorem}\label{theorem1}
Global context aware transform that is distinct for every relation for learning relation and entity embeddings in separate vector spaces is strictly more expressive than i) Learning the same embedding space ii) Using a common transform for every relation iii) Using local context only
\end{theorem}
\begin{proof}
Follows from lemma \ref{lemma1} to \ref{lemma6}
\end{proof}

\begin{theorem}\label{algo1_t1}
 There exists an optimum point for the ranking loss between the triplet vector additions of positive and negative triples, which can be traversed with decreasing loss at each step of the optimization from any point in the embedding space, and as such, an optimum optimization algorithm should be able to find such a point
\end{theorem}
\begin{proof}
\par \noindent Let us define the framework of the ranking loss as below.
\par \noindent Consider a positive triple ($e_1,r,e_2$) and a negative triple ($e_3,r,e_4$). The vector addition for the first triple would give $t_1=norm(\vec{e}_1+\vec{r}-\vec{e}_2)$ and for the second would give $t_2=norm(\vec{e}_3+\vec{r}-\vec{e}_4)$. The margin loss would then be defined as $max(0,margin-(t_2-t_1))$.
\par \noindent If we take the margin to be zero and ignore the term $t_2$ we get $loss=max(0,t_1)$. Since the norm has to be $>=0, t1>=0$, hence, the loss becomes minimum when $t_1=0$. Removing the trivial case of all entity embeddings=$\vec{0}$, we define the loss space as follows. Without loss of generality we take the relation vectors to be fixed. For a triple ($\vec{e}_1,\vec{r},\vec{e}_2$) we take the difference $e_2-e_1$. The loss for this triple then becomes $r-(e_2-e_1)$. For all triples, we get 

\begin{equation}\label{eq_loss_ranking}
\small
\begin{aligned}
Loss &= \sum_{i\in \mathcal{T}}\left(r^i - (e_2^i-e_1^i)\right)= \sum_{i\in \mathcal{T}}\left(r^i\right)  - \sum_{i\in \mathcal{T}}\left(e_2^i-e_1^i\right)
\end{aligned}
\end{equation}

\par \noindent Now we define the point in vector space represented by $\sum_{i \in \mathcal{T}}(e_2^i-e_1^i)$ to be the current point in the optimization and plot the loss concerning it, which is the norm of the loss in the equation \ref{eq_loss_ranking}. Since there could be multiple configurations of the entity embeddings for each such point, we assume the loss to be an optimum loss given a configuration of entity embeddings. I.e., the relation vectors to be modified such that each difference term $r-(e_2-e_1)$ is always greater than or equal to 0.
\par \noindent Let  $R=\sum_{i \in Triples}r_i$ and  $E=\sum_{i \in Triples}(e_2^i-e_1^i)$, then $Loss = \mid R-E \mid $ represents a cone. Now if we consider all the possible relation vector configurations and take all losses so that at each point in the vector space the minimum of each contribution is taken we get a piece-wise continuous function with conical regions and hyperbolic intersection of the cones as in figure \ref{loss_topology}.
\par \noindent For a path to exist between the start and an optimum global point under gradient descent, two conditions must hold\:
\begin{enumerate}
    \item The function must be continuous.
    \item At no point in the function must there be a point such that there exists no point in it’s neighborhood with a lesser value.
\end{enumerate}
The derived function satisfies both the above properties.
\end{proof}




\begin{figure}[h!]
  \includegraphics[scale=0.30]{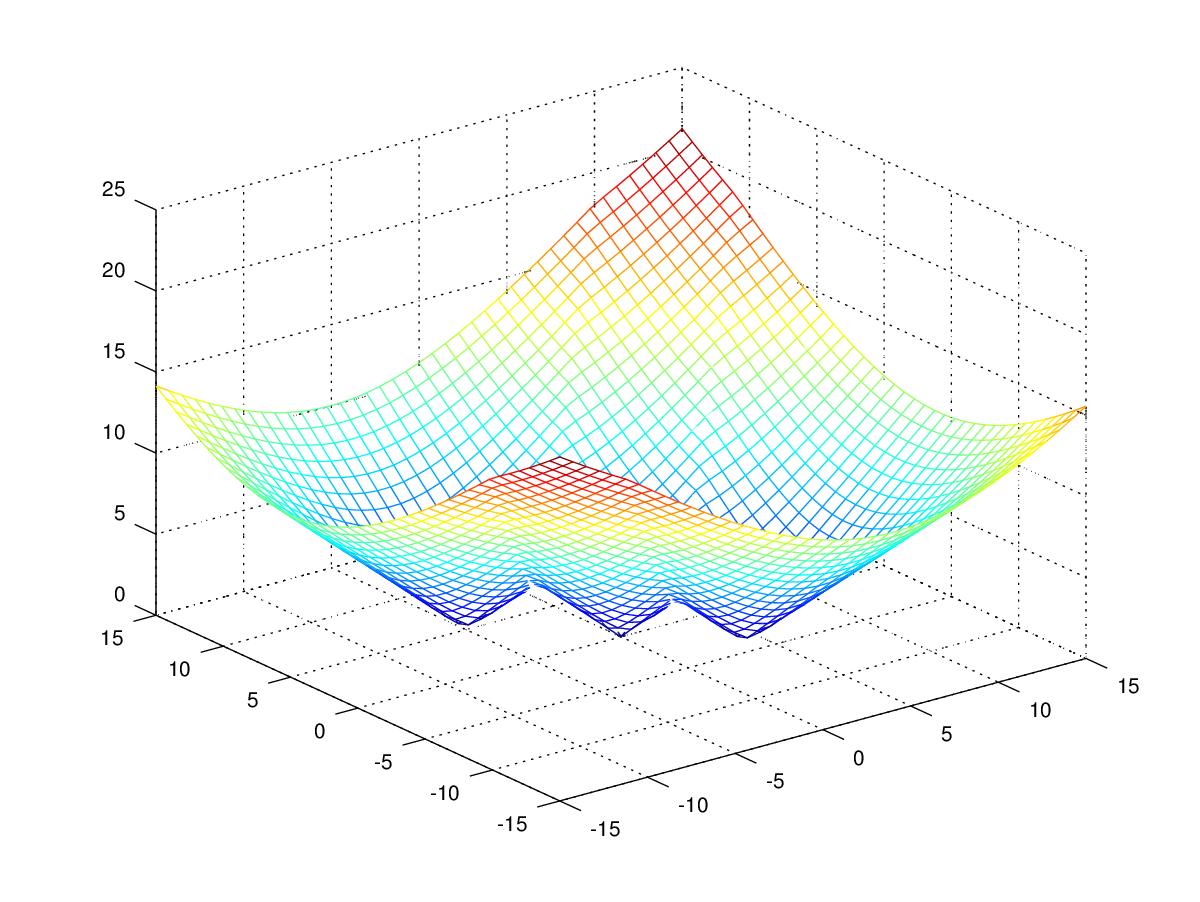}
  \caption{Loss function topology under the $l_1$ norm of the difference between the sum of relation vectors and entity vectors, demonstrating that convergence is possible from any starting point}
  \label{loss_topology}
  \vspace{-2mm}
\end{figure}
The above theorem proves convergence when all entities are updated simultaneously. However, this may not be possible in practice as the number of entities could be very large, causing memory errors. We introduce a simple modification to train the entities batch-wise, i.e., to update via gradient descent only a sample of the entities, thus reducing memory requirements. We shall see in the next theorem that this approach also converges.
\begin{theorem}\label{algo1_t2}
The entity vectors could be updated batch wise to monotonically reduce the loss till optimum is reached
\end{theorem}
\begin{proof}
\begin{algorithm}
\label{algo_1}
\small
\SetAlgoNoLine
 Initialize the relation and entity embeddings randomly; \leavevmode\newline 
 \While{\textbf{not} converged}{
    \begin{enumerate}[leftmargin=*,label=$\bullet$,noitemsep,partopsep=0pt,topsep=0pt,parsep=0pt]
        \item Select a subset of entities
        \leavevmode\newline
        $\{e_1,e_2...e_n\} \subseteq E$
        
        \item Select the subset of 1-hop \& 2-hop triples $\mathcal{T}_{batch} \subseteq \mathcal{T} \mid e \in \uptau \land \uptau \in \mathcal{T}_{batch} \land e \in \{e_1,e_2...e_n\}$ 
        
        \item Input $\mathcal{T}$ to KGGAT-SEP model and compute a  forward pass to get the new  entity embeddings for  the entities in the current batch keeping the other entity embeddings fixed.
        
        \item Compute the loss according to $\small L(\Omega) =  \sum_{\uptau_{ht} \in \mathcal{T}^{pos}}\sum_{\uptau_{ht}^{'} \in \mathcal{T}^{neg}}max\{d_{\uptau_{ht}^{'}} - d_{\uptau_{ht}}+ \gamma, 0 \}$
        
        \item Back propagate using gradient descent to update $\{e_1,e_2...e_n\} \subseteq E$
    \end{enumerate}
 }
 \caption{Algorithm for learning entity embeddings batchwise using the margin ranking loss}
\end{algorithm}

Consider a set of vectors $\vec{e}_1,\vec{e}_2...\vec{e}_n$ and the resultant $\vec{r}$.
\[\vec{r} = \vec{e}_1 + \vec{e}_2 + ... + \vec{e}_n\]
Also consider another set of entities $\vec{e}_1^{~’},\vec{e}_2^{~’}...\vec{e}_n^{~’}$.
The difference between $\vec{r}$ and the sum of new set of vectors is
\begin{align*}
\vec{d} &= \vec{r}-(\vec{e}_1^{~’}+\vec{e}_2^{~’}+.....+\vec{e}_n^{~’}) \\
&=(\vec{e}_1-\vec{e}_1^{~’}) + …… + (\vec{e}_n-\vec{e}_n^{~’})
\end{align*}
\par \noindent Now if we update a vector $\vec{e}_i^{~’}$  to $\vec{e}_i^{~”}$ to be closer to $\vec{e}_i$ such that \[\mid \vec{e}_i-\vec{e}_i^{~’} \mid >= \mid \vec{e}_i-\vec{e}_i^{~”} \mid\]
Then,
\begin{align*}
&\mid \vec{r}-(\vec{e}_1^{~’}+....+\vec{e}_i^{~’}+....+\vec{e}_n^{~’}) \mid>=\\
&\mid \vec{r}-(\vec{e}_1^{~’}+....+\vec{e}_i^{~”}+....+\vec{e}_n^{~’}) \mid 
\end{align*}
\par \noindent Theorem \ref{algo1_t1} shows that such an update exists and performing it recursively for other entity vectors till optimum is possible under the given framework. Algorithm 1 details the batch wise learning.

\end{proof}

\end{document}